\documentclass[11pt,a4paper]{article}
\usepackage{hyperref}
\usepackage{cleveref}
\usepackage[hyperref]{acl2020}
\usepackage{times}
\usepackage{graphicx}
\usepackage{latexsym}

\usepackage{amsmath}
\usepackage{amssymb}
\usepackage{color}
\usepackage{amsthm}

\newtheorem{proposition}{Proposition}
\newtheorem{corollary}{Corollary}
\newtheorem{lemma}{Lemma}
\newtheorem{definition}{Definition}
\newtheorem*{example}{Example}
\newtheorem*{proposition2}{Proposition}

\usepackage{microtype}

\aclfinalcopy 

\title{Emergence of Syntax Needs Minimal Supervision}

\author{Rapha\"el Bailly  \\
  SAMM, EA 4543, FP2M 2036 CNRS\\
  Universit\'e Paris 1 Panthéon-Sorbonne\\
  \texttt{raphael.bailly@univ-paris1.fr} \\\And
  Kata G\'abor  \\
  ERTIM, EA 2520 \\ INALCO \\
  \texttt{kata.gabor@inalco.fr} \\}

\date{}

\begin{document}
\maketitle
\begin{abstract}
This paper is a theoretical contribution to the debate on the learnability of syntax from a corpus without explicit syntax-specific guidance.
 Our approach originates in the observable structure of a corpus, which we use to define and isolate grammaticality (syntactic information) and meaning/pragmatics information. We describe the formal characteristics of an autonomous syntax and show that it becomes possible to search for syntax-based lexical categories with a simple optimization process, without any prior hypothesis on the form of the model.
\end{abstract}

\section{Introduction}
 
Syntax is the essence of human linguistic capacity that makes it possible to produce and understand a potentially infinite number of unheard sentences. The principle of compositionality \cite{frege} states that the meaning of a complex expression is fully determined by the meanings of its constituents \textit{and its structure}; hence, our understanding of sentences we have never heard before comes from the ability to construct the sense of a sentence out of its parts.
The number of constituents and assigned meanings is necessarily finite. Syntax is responsible for creatively combining them, and it is commonly assumed that syntax operates by means of algebraic compositional rules \cite{chomsky57} and a finite number of syntactic categories.

One would also expect a computational model of language to have - or be able to acquire - this compositional capacity. The recent success of neural network based language models on several NLP tasks, together with their "black box" nature, attracted attention to at least two questions. First, when recurrent neural language models generalize to unseen data, does it imply that they acquire syntactic knowledge, and if so, does it translate into human-like compositional capacities \cite{baroni19,lake17,linzen16,gulordava18}? Second, whether research into neural networks and linguistics can benefit each other \cite{pater,replypater}; by providing evidence that syntax can be learnt in an unsupervised fashion \cite{blevins18}, or the opposite, humans and machines alike need innate constraints on the hypothesis space (a universal grammar) \cite{kucoro18,vanschijndel19}?

A closely related question is whether it is possible to learn a language's syntax exclusively from a corpus. The \textit{poverty of stimulus} argument \cite{chomsky80} suggests that humans cannot acquire their target language from only positive evidence unless some of their linguistic knowledge is innate. The machine learning equivalent of this categorical "no" is a formulation known as \textit{Gold's theorem} \cite{gold67}, which suggests that the complete unsupervised learning of a language (correct grammaticality judgments for every sequence), is intractable from only positive data. \citet{clark10} argue that Gold's paradigm does not resemble a child's learning situation and there exist algorithms that can learn unconstrained classes of infinite languages \cite{clark06}. This ongoing debate on syntax learnability and the poverty of the stimulus can benefit from empirical and theoretical machine learning contributions \cite{lappin07,mccoy18,linzen19}.

In this paper, we argue that syntax can be inferred from a sample of natural language with very minimal supervision. We introduce an information theoretical definition of what constitutes syntactic information. The linguistic basis of our approach is the autonomy of syntax, which we redefine in terms of (statistical) independence.  We demonstrate that it is possible to establish a syntax-based lexical classification of words from a corpus without a prior hypothesis on the form of a syntactic model. 

Our work is loosely related to previous attempts at optimizing language models for syntactic performance \cite{dyer16,kucoro18} and more particularly to \citet{li19} because of their use of mutual information and the information bottleneck principle \cite{tishby00}. However, our goal is different in that we demonstrate that very minimal supervision is sufficient in order to guide a symbolic or statistical learner towards grammatical competence.


\section{Language models and syntax}\label{deeplms}

As recurrent neural network based language models started to achieve good performance on different tasks \cite{mikolov10}, this success sparked attention on whether such models implicitly learn syntactic information. Language models are typically evaluated using perplexity on test data that is similar to the training examples. However, lower perplexity does not necessarily imply better \textit{syntactic} generalization. Therefore, new tests have been put forward to evaluate the linguistically meaningful knowledge acquired by LMs. 

A number of tests based on artificial data have been used to detect compositionality or systematicity in deep neural networks. \citet{lake17} created a task set that requires executing commands expressed in a compositional language. \citet{bowman15} design a task of logical entailment relations to be solved by discovering a recursive compositional structure. \citet{saxton19} propose a semi-artificial probing task of mathematics problems.

\citet{linzen16} initiated a different line of linguistically motivated evaluation of RNNs. Their data set consists in minimal pairs that differ in grammaticality and instantiate sentences with long distance dependencies (e.g. number agreement). The model is supposed to give a higher probability to the grammatical sentence. The test aims to detect whether the model can solve the task even when this requires knowledge of a hierarchical structure. 
Subsequently, several alternative tasks were created along the same concept to overcome specific shortcomings \cite{bernardy17,gulordava18}, or to extend the scope to different languages or phenomena \cite{ravfogel18,ravfogel19}.

It was also suggested that the information content of a network can be tested using "probing tasks" or "diagnostic classifiers" \cite{giulianelli18,hupkes18}. This approach consists in extracting a representation from a NN and using it as input for a supervised classifier to solve a different linguistic task. 
Accordingly, probes were conceived to test if the model learned parts of speech \cite{saphra18}, morphology \cite{belinkov17,peters18}, or syntactic information. \citet{tenney19} evaluate contextualized word representations on syntactic and semantic sequence labeling tasks. Syntactic knowledge can be tested by extracting constituency trees from a network's hidden states \cite{peters18b} or from its word representations \cite{hewitt19a}. Other syntactic probe sets include the work of \citet{conneau-etal-2018-cram} and \citet{marvin18}.

Despite the vivid interest for the topic, no consensus seems to unfold from the experimental results. Two competing opinions emerge:

\begin{itemize}
    \item Deep neural language models generalize by learning human-like syntax: given sufficient amount of training data, RNN models approximate human compositional skills and implicitly encode hierarchical structure at some level of the network. This conjecture coincides with the findings of, among others \citet{bowman15,linzen16,giulianelli18,gulordava18,kucoro18}.
    
    \item The language model training objective does not allow to learn compositional syntax from a corpus alone, no matter what amount of training data the model was exposed to. Syntax learning can only be achieved with task-specific guidance, either as explicit supervision, or by restricting the hypothesis space to hierarchically structured models \cite{dyer16,marvin18,chowdhury18,vanschijndel19,lake17}.  
\end{itemize}

Moreover, some shortcomings of the above probing methods make it more difficult to come to a conclusion. Namely, it is not trivial to come up with minimal pairs of naturally occurring sentences that are equally likely. Furthermore, assigning a (slightly) higher probability to one sentence does not reflect the nature of knowledge behind a grammaticality judgment. Diagnostic classifiers may do well on a linguistic task because they learn to solve it, not because their input contains a hierarchical structure \cite{hewitt19c}. In what follows, we present our assessment on how the difficulty of creating a linguistic probing data set is interconnected with the theoretical problem of learning a model of syntactic competence.

\subsection{Competence or performance, or why syntax drowns in the corpus}
\label{performance}

If syntax is an autonomous module of linguistic capacity, the rules and principles that govern it are formulated independently of meaning. However, a corpus is a product of language use or \textit{performance}. Syntax constitutes only a subset of the rules that generate such a product; the others include communicative needs and pragmatics. Just as meaning is uncorrelated with grammaticality, corpus frequency is only remotely correlated with human grammaticality judgment \cite{newmeyer}. 

Language models learn a probability distribution over sequences of words. The training objective is not designed to distinguish grammatical from agrammatical, but to predict language use. 
While \citet{linzen16} found a correlation between the perplexity of RNN language models and their syntactic knowledge, subsequent studies \cite{bernardy17,gulordava18} recognized that this result could have been achieved by encoding lexical semantic information, such as argument typicality. E.g. "in '\emph{dogs (...) bark}', an RNN might get the right agreement by encoding information about what typically barks" \cite{gulordava18}.

Several papers revealed the tendency of deep neural networks to fixate on surface cues and heuristics instead of "deep" generalization in solving NLP tasks  \cite{levy15,niven19}. In particular, \citet{mccoy19} identify three types of syntactic heuristics that get in the way of meaningful generalization in  language models. 

Finally, it is difficult to build a natural language data set without semantic cues. Results from the syntax-semantics interface research show that lexical semantic properties account for part of syntactic realization \cite{levin05}.  


\section{What is syntax a generalization of?}
\label{section-illustr}

We have seen in section \ref{deeplms} that previous works on the linguistic capacity of neural language models concentrate on compositionality, the key to creative use of language. However, this creativity is not present in language models: they are bound by the type of the data they are exposed to in learning. 

We suggest that it is still possible to learn syntactic generalization from a corpus, but not with likelihood maximization. We propose to isolate the syntactic information from shallow performance-related information. In order to identify such information without explicitly injecting it as direct supervision or model-dependent linguistic presuppositions, we propose to examine \textit{inherent structural properties} of corpora. As an illustration, consider the following natural language sample:
\begin{center}
 \textit{cats eat rats\\
rats fear cats\\
mathematicians prove theorems\\
doctors heal wounds\\}
\end{center}

According to the Chomskyan principle of the \textit{autonomy of syntax} \cite{chomsky57}, the syntactic rules that define well-formedness can be formulated without reference to meaning and pragmatics. For instance, the sentence \textit{Colorless green ideas sleep furiously} is grammatical for humans, despite being meaningless and unlikely to occur. We study whether it is possible to deduce, from the \textit{structural properties} of our sample above, human-like grammaticality judgments that predict sequences like \textit{cats rats fear} as agrammatical, and accept e.g. \textit{wounds eat theorems} as grammatical.

We distinguish two levels of observable structure in a corpus:

1. the proximity; the tendency of words to occur in the context of each other (in the same document/same sentence, etc.)

2. the order in which the words appear.

\begin{definition}
Let $L$ be a language over vocabulary $V$.  The language that contains every possible sequence obtained by shuffling the elements in a sequence of $L$ will be denoted  $\overline{L}$.
\end{definition}

If $V^*$ is the set of every possible sequence over vocabulary $V$ and $L$ is the language instantiated by our corpus, L is generated by a mixture of contextual and syntactic constraints over $V^*$. We are looking to separate the syntactic specificities from the grammatically irrelevant, contextual cues. The processes that transform $V^*$ into $\overline{L}$, and $\overline{L}$ into $L$

$$V^* \xrightarrow{\text{proximity}} \overline{L} \xrightarrow{\text{order}} L$$

\noindent are entirely dependent on words: it should be possible to encode the information used by these processes into word categories.

In what follows, we will provide tools to isolate the information involved in proximity from the information involved in order. We also relate these categories to linguistically relevant concepts. 



\subsection{Isolating syntactic information}
For a given word, we want to identify the information involved in each type of structure of the corpus, and represent it as partitions of the vocabulary into lexical categories: 

1. \textbf{Contextual} information is any information unrelated to sentence structure, and hence, grammaticality: this encompasses meaning, topic, pragmatics, corpus artefacts etc. The surface realization of sentence structure is a language-specific combination of word order and morphological markers. 

2. \textbf{Syntactic} information is the information related to sentence structure and - as for the autonomy requirement - nothing else: it is independent of all contextual information.

In the rest of the paper we will concentrate on English as an example, a language in which syntactic information is primarily encoded in order. In section \ref{concl} we present our ideas on how to deal with morphologically richer languages.  

\begin{definition}
Let $L$ be a language over vocabulary $V = \{ v_1, \dots \}$, and $P = (V,C, \pi: V \mapsto C)$ a partition of $V$ into categories $C$.  Let $\pi(L)$ denote the language that is created by replacing a sequence of elements in $V$ by the sequence of their categories.

One defines the partition $P_{tot} = \{ \{v\}, v \in V$\} (one category per word) and the partition $P_{nul}= \{V\}$ (every word in the same category).
\end{definition}

$P_{tot}$ is such that $\pi_{tot}(L) \sim L$. The minimal partition $P_{nul}$  does not contain any information.\\

A partition $P = (V,C,\pi)$ is contextual if it is impossible to determine word order in language $L$ from sequences of its categories:

\begin{definition}
Let $L$ be a language over vocabulary $V$, and let $P = (V,C,\pi)$ be a partition over $V$. The partition $P$ is said to be \emph{contextual} if $$\pi(L) = \pi(\overline{L}) $$
\end{definition}

The trivial partition $P_{nul}$ is always contextual. 

\begin{example}
Consider the natural language sample. We refer to the words by their initial letters: \textit{r(ats),e(at)...}, thus we have $V = \{c,e,r,f,m,p,t,d,h,w\}$. and $L=\{cer,rfc,mpt,dhw \}$. 

  One can check that the partition  $P_1$ : $$c_1=\{c, r, e, f\}$$ $$ c_2=\{m,p,t\}$$ $$c_3=\{d,h,w\}$$  is contextual:  the well-formed sequences over this partition are $c_1 c_1 c_1$, $c_2 c_2 c_2$ and $c_3 c_3 c_3$. These patterns convey the information that words like \textit{'mathematicians'} and \textit{'theorems'} occur together, but do not provide information on order. Therefore $\pi_1(L) = \{c_1 c_1 c_1, c_2 c_2 c_2, c_3 c_3 c_3\} = \pi_1(\overline{L})$. 
  $P_1$ is also a \textit{maximal} partition for that property: any further splitting leads to order-specific patterns. Intuitively, this partition corresponds to the semantic categories $ Animals = \{r,c,e,f\}$, $Science = \{m,p,t\}$, and $Medicine = \{d,h,w\}$.
\end{example}

A syntactic partition has two characteristics: its patterns encode the structure (in our case, order), and it is completely autonomous with respect to contextual information. Let us now express this autonomy formally.\\
Two partitions of the same vocabulary are said to be independent if they do not share any information with respect to language $L$. In other words, if we translate a sequence of symbols from $L$ into their categories from one partition, this sequence of categories will not provide any information on how the sequence translates into categories from the other partition: 

\begin{definition}
Let $L$ be a language over vocabulary $V$, and let $P=(V,C,\pi)$ and $P' = (V,C', \pi')$ be two partitions of $V$. $P$ and $P'$ are considered as independent with respect to $L$ if  $$\forall c_{i_1} \dots c_{i_n} \in \pi(L), \forall c'_{j_1} \dots c'_{j_n} \in \pi'(L)$$ $$  \pi^{-1} (c_{i_1} \dots c_{i_n} ) \cap {\pi'}^{-1} (c'_{j_1} \dots c'_{j_n} ) \neq \emptyset$$
\end{definition}

\begin{definition}
\label{syntacticpartition}
Let $L$ be a language over $V$, and let $P=(V,C,\pi)$ be a partition. $P$ is said to be \emph{syntactic} if it is independent of any contextual partition of $V$.
\end{definition}

A syntactic partition is hence a partition that does not share any information with contextual partitions; or, in linguistic terms, a syntactic pattern is equally applicable to any contextual category.

\begin{example}
We can see that the partition $P_{2}$ : $$c_4=\{c, r, m, t, d, w\}$$ $$ c_5=\{e,f,p,h\}$$ is independent of the partition $P_{1}$: one has $\pi_2(L)=\{ c_4 c_5 c_4\}$. Knowing the sequence $c_4 c_5 c_4$ does not provide any information on which $P_{1}$ categories the words belong to. $P_2$ is therefore a syntactic partition.

\end{example}


Looking at the corpus, one might be tempted to consider a partition $P_3$ that sub-divides $c_4$ into subject nouns, object nouns, and - if one word can be mapped to only one category - "ambiguous" nouns: $$c_6=\{m, d\}$$ $$c_7=\{t,w\}$$ $$ c_8=\{c, r\}$$ $$ c_9=\{e,f,p,h\}$$ 
The patterns corresponding to this partition would be
 $\pi_3(L) = \{c_6 c_9 c_7, c_8 c_9 c_8\}$. These patterns will not predict that sentence (2) is grammatical, because the word \emph{wounds} was only seen as an object. If we want to learn the correct generalization we need to reject this partition in favour of $P_2$.\\
This is indeed what happens by virtue of definition \ref{syntacticpartition}. We notice that the patterns over $P_3$ categories are not independent of the contextual  partition $P_1$: one can deduce from the rule $c_8 c_9 c_8$ that the corresponding sentence cannot be e.g. category $c_2$: 

$$\pi_3^{-1}(c_8 c_9 c_8) \cap \pi_1^{-1}(c_2 c_2 c_2)=\emptyset$$

\noindent $P_3$ is hence rejected as a syntactic partition. 

$P_2$ is the maximal syntactic partition: any further distinction that does not conflate $P_1$ categories would lead to an inclusion of contextual information. We can indeed see that category $c_4$ corresponds to \emph{Noun} and $c_5$ corresponds to \emph{Verb}. The syntactic rule for the sample is \emph{Noun} \emph{Verb} \emph{Noun}. It becomes possible to distinguish between syntactic and contextual acceptability: \emph{cats rats fear} is acceptable as a contextual pattern $c_1 c_1 c_1$ under \emph{'Animals'}, but not a valid syntactic pattern. The sequence \emph{wounds eat theorems} is syntactically well-formed by $c_5 c_6 c_5$, but does not correspond to a valid contextual pattern.

In this section we provided the formal definitions of syntactic information and the broader contextual information. By an illustrative example we gave an intuition of how we apply the autonomy of syntax principle in a non probabilistic grammar. We now turn to the probabilistic scenario and the inference from a corpus.

 \section{Syntactic and contextual categories in a corpus}
 
As we have seen in section \ref{deeplms}, probabilistic language modeling with a likelihood maximization objective does not have incentive to concentrate on syntactic generalizations. In what follows, we demonstrate that using the autonomy of syntax principle it is possible to infer syntactic categories for a probabilistic language.

A stochastic language $L$ is a language which assigns a probability to each sequence. As an illustration of such a language, we consider the empirical distribution induced from the sample in section \ref{section-illustr}. $$L=\{cer(\frac{1}{4}), rfc(\frac{1}{4}), mpt(\frac{1}{4}),dhw(\frac{1}{4})\}$$ We will denote by $p_L(v_{i_1} \dots v_{i_n})$ the probability distribution associated to $L$.
 
 \begin{definition}
 Let $V$ be a vocabulary. A \emph{(probabilistic) partition} of $V$ is defined by $P = (V,C,\pi : V \mapsto \mathbb{P}(C))$ where $\mathbb{P}(C)$ is the set of probability distributions over $C$.
 \end{definition}

\begin{example}
The following probabilistic partitions correspond to the non-probabilistic partitions (contextual and syntactic, respectively) defined in section \ref{section-illustr}. We will now consider these partitions in the context of the probabilistic language $L$.

\bigskip
\begin{center}
$\pi_1 =   \begin{smallmatrix} c\\ r\\ e\\ f \\ m \\ p \\ t \\ d\\ h \\w \end{smallmatrix}  \left( \begin{smallmatrix} 1&0&0 \\1&0&0 \\1&0&0 \\1&0&0 \\0&1&0 \\0&1&0 \\0&1&0 \\0&0&1 \\0&0&1 \\0&0&1 \end{smallmatrix}\right),  \pi_2 =   \begin{smallmatrix} c\\ r\\ e\\ f \\ m \\ p \\ t \\ d\\ h \\w \end{smallmatrix}  \left( \begin{smallmatrix} 1&0 \\1&0 \\0&1 \\0&1 \\1&0 \\0&1 \\1&0 \\1&0 \\0&1 \\1&0 \end{smallmatrix}\right)  $
\end{center}
\bigskip

\end{example}

From a probabilistic partition $P = (V,C,\pi)$ as defined above, one can map a stochastic language $L$ to a stochastic language $\pi(L)$ over the sequences of categories:
$$p_{\pi}(c_{i_1} \dots c_{i_n}) =$$ $$ \sum_{u_{j_1} \dots u_{j_n}} (\prod_k \pi(c_{i_k} | u_{j_k}) ) p_L(u_{j_1} \dots u_{j_n}) $$

As in the non-probabilistic case, the language $\overline{L}$ will be defined as the language obtained by shuffling the sequences in $L$.

\begin{definition}
Let $L$ be a stochastic language over vocabulary $V$. We will denote by $\overline{L}$ the language obtained by shuffling the elements in the sequences of $L$ in the following way: for a sequence $v_1 \dots v_n$, one has $$p_{\overline{L}}(v_1 \dots v_n) = \frac{1}{n!}\sum_{(i_1 \dots i_n) \in \sigma(n)} p_{L}(v_{i_1} \dots v_{i_n})$$
\end{definition}

One can easily check that $\pi(\overline{L})= \overline{\pi(L)}$.

\begin{example}
The stochastic patterns of $L$ over the two partitions are, respectively: $$\pi_1(L)= \{c_1 c_1 c_1(\frac{1}{2}), c_2 c_2 c_2(\frac{1}{4}), c_3 c_3 c_3(\frac{1}{4})\}$$ $$\pi_2(L)= \{c_4 c_5 c_4(1)\}$$
\end{example}

We can now define a probabilistic contextual partition:

\begin{definition}
Let $L$ be a stochastic language over vocabulary $V$, and let $P = (V,C,\pi)$ be a probabilistic partition. $P$ will be considered as \emph{contextual} if $$\pi(L) = \pi(\overline{L})$$
\end{definition}

We now want to express the independence of syntactic partitions from contextual partitions. The independence of two probabilistic partitions can be construed as an independence between two random variables:

\begin{definition}\label{def_ind_part}
Consider two probabilistic partitions $P = (V,C,\pi)$ and $P' = (V,C',\pi')$. We will use the notation  $$(\pi \cdot \pi')_v (c_i , c'_j) = \pi_v(c_i) \pi'_v(c'_j)$$ and the notation $$P\cdot P'= (V, C\times C', \pi \cdot \pi')$$
 $P$ and $P'$ are said to be independent (with respect to $L$) if the distributions inferred over sequences of their categories are independent:
 $$\forall w \in \pi(L), \forall w' \in \pi'(L),$$
 $$p_{\pi \cdot \pi'}(w,w') =p_\pi (w) p_{\pi'}(w') $$
\end{definition}

A syntactic partition will be defined by its independence from contextual information:

\begin{definition}
Let $P$ be a probabilistic partition, and $L$ a stochastic language. The partition $P$ is said to be \emph{syntactic} if it is independent (with respect to $L$) of any possible probabilistic contextual partition in $L$.
\end{definition}

\begin{example}
The partition $P_{1}$ is contextual, as $\pi_{1}(L) = \overline{\pi_{1}({L})}$. The partition $P_{2}$ is clearly independent of $P_1$ w.r.t. L.
\end{example}

\subsection{Information-theoretic formulation}

The definitions above may need to be relaxed if we want to infer syntax from natural language corpora, where strict independence cannot be expected. We propose to reformulate the definitions of contextual and syntactic information in the information theory framework. 

\smallskip

We present a relaxation of our definition based on Shannon's information theory \cite{shannon}. We seek to quantify the amount of information in a partition $P=(V,C,\pi)$ with respect to a language $L$. Shannon's entropy provides an appropriate measure. Applied to $\pi(L)$, it gives $$H(\pi(L))= - \sum_{w \in \pi(L)} p_\pi(w)(\log(p_\pi(w)))$$

For a simpler illustration, from now on we will consider only languages composed of fixed-length sequences $s$, i.e $|s|=n$ for a given $n$. If $L$ is such a language, we will consider the language $\overline{\overline{L}}$ as the language of sequences of size $n$ defined by $$p_{\overline{\overline{L}}} (v_{i_1} \dots v_{i_n}) = \prod_j p_L(v_{i_j})$$
where $p_L(v)$ is the frequency of $v$ in language $L$. 

\begin{proposition}\label{prop1}
Let $L$ be a stochastic language, $P=(V,C,\pi)$ a partition. One has:
$$H(\pi(\overline{\overline{L}})) \geq H(\pi(\overline{L})) \geq H(\pi(L))$$
with equality iff the stochastic languages are equal.
\end{proposition}

Let $C$ be a set of categories. For a given distribution over the categories $p(c_i)$, the partition defined by $\pi(c_i | v) = p(c_i)$ (constant distribution w.r.t. the vocabulary) contains no information on the language. One has $p_{\pi}(c_{i_1} \dots c_{i_k})= p(c_{i_1}) \dots p(c_{i_k})$, which is the unigram distribution, in other words $\pi(L)= \pi(\overline{\overline{L}})$. As the amount of syntactic or contextual information contained in $\overline{\overline{L}}$ can be considered as zero, a consistent definition of the information would be:

\begin{definition}
Let $P=(V,C,\pi)$ be a partition, and $L$ a language. The information contained in $P$ with respect to $L$ is defined as $$I_L(P) = H(\pi(\overline{\overline{L}})) - H(\pi(L))$$
\end{definition}

\begin{lemma}
Information $I_L(P)$ defined as above is always positive. One has $I_{\overline{L}}(P) \leq I_L(P)$, with equality iff $\pi(\overline{L}) = \pi(L)$.
\end{lemma} 
\smallskip

After having defined how to measure the amount of information in a partition with respect to a language, we now translate the independence between two partitions into the terms of mutual information:

\begin{definition}\label{def:mutual-info}
We follow notations from Definition~\ref{def_ind_part}. We define the mutual information of two partitions $P = (V,C,\pi)$ et $P' = (V,C',\pi')$ with respect to $L$ as $$I_{L}(P; P') = H(P)+ H(P')- H(P \cdot P')$$
\end{definition}

This directly implies that
\begin{lemma}
$P = (V,C,\pi)$ and $P' = (V,C',\pi')$ are independent w.r.t. $L$
$$\Leftrightarrow I_{L}(P; P') = 0$$
\end{lemma}
\begin{proof}
This comes from the fact that, by construction, the marginal distributions of $\pi \cdot \pi'$ are the distributions $\pi$ and $\pi'$.
\end{proof}

With these two definitions, we can now propose an information-theoretic reformulation of what constitutes a contextual and a syntactic partition:
\begin{proposition}\label{propXX} 
Let $L$ be a stochastic language over vocabulary $V$, and let $P = (V,C,\pi)$ be a probabilistic partition.
\begin{itemize}
\item $P$ is \emph{contextual} iff
$$I_L(P) = I_{\overline{L}}(P)$$
\item $P$ is \emph{syntactic} iff for any contextual partition $P_*$
$$I_{L}(P; P_*) = 0$$
\end{itemize} 
\end{proposition}

\subsection{Relaxed formulation}

If we deal with non artificial samples of natural language data, we need to prepare for sampling issues and word (form) ambiguities that make the above formulation of independence too strict. Consider for instance adding the following sentence to the previous sample:

\begin{center}
\textit{doctors heal fear\\}
\end{center}

The distinction between syntactic and contextual categories is not as clear as before. We need a relaxed formulation for real corpora: we introduce $\gamma$-contextual and $\mu, \gamma$-syntactic partitions.

\begin{definition}
Let $L$ be a stochastic language.
\begin{itemize}
    \item A partition $P$ is considered as \emph{ $\gamma$-contextual} if it minimizes 
\begin{equation} \label{eq:1}
I_{L}(P)(1-\gamma) - I_{\overline{L}}(P)
\end{equation}
\item  A partition $P$ is considered \emph{$\mu, \gamma$-syntactic} if it minimizes 
\begin{equation} \label{eq:2}
\max_{P^*} I_{L}(P; P_{*}) - \mu~I_{L}(P)
\end{equation}
for any $\gamma$-contextual partition $P^*$. 
\end{itemize}
\end{definition}

Let $P$ and $P'$ be two partitions for $L$, such that $$\Delta_I(L) = I_{P'}(L) - I_P(L) \geq 0$$ then the $\gamma$-contextual program (\ref{eq:1}) would choose $P'$ over $P$ iff $$\frac{\Delta_I(L)-\Delta_I(\overline{L})}{\Delta_I(L)} \leq \gamma$$

Let $P^*$ be a $\gamma$-contextual partition. Let $$\Delta_{MI}(L,P^*) =  I_L({P'; P^*})- I_L({P; P^*})$$
then the $\mu, \gamma$-syntactic program (\ref{eq:2}) would choose $P'$ over $P$ iff $$\frac{\Delta_{MI}(L,P^*)}{\Delta_I(L)} \leq \mu$$
\begin{example}
Let us consider the following partitions:
\noindent - $P_1$ and $P_2$ refer to the previous partitions above: \{Animals, Science, Medicine\} and \{Noun, Verb\}

\begin{figure}[!b]
\centering
\includegraphics[scale=0.5]{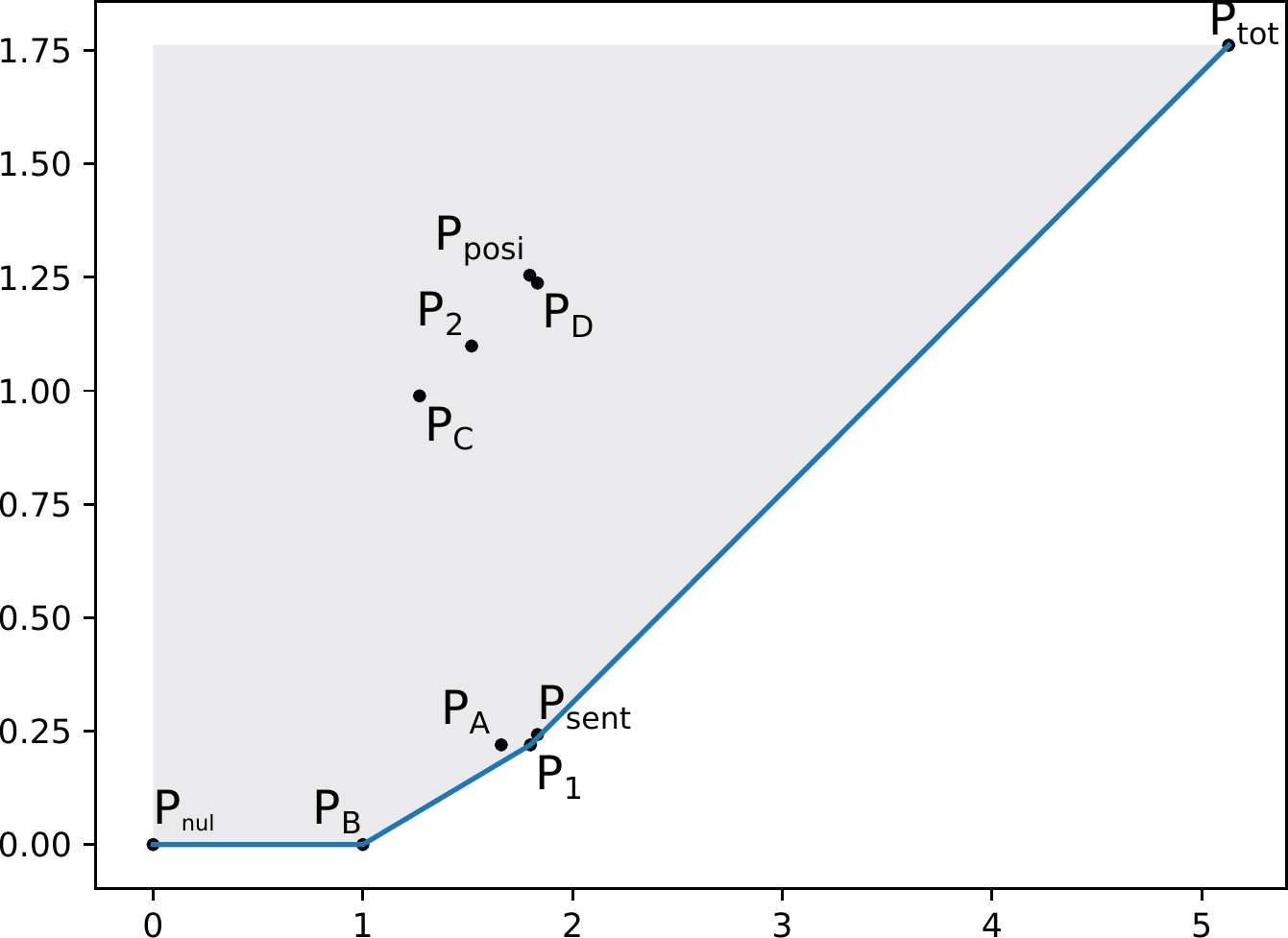}
\caption{$I_{L}(P)-I_{\overline{L}}(P)$ represented w.r.t. $I_{L}(P)$ for different partitions: acceptable solutions of program~(\ref{eq:1}) lie on the convex hull boundary of the set of all partitions. Solution for $\gamma$ is given by the tangent of slope $\gamma$. Non trivial solutions are $P_B$ and $P_1$.}
\label{fig_deltaI}
\end{figure}

\noindent - $P_A$ is adapted from $P_1$ so that \textit{'fear'} belongs to \textit{Animals} and \textit{Medicine}

\begin{center}
$\{c,e,r,f(\frac{1}{2})\}, \{m,p,t\}, \{d,h,w,f(\frac{1}{2})\}$
\end{center}

\noindent - $P_B$ merges \textit{Animals} and \textit{Medicine} from $P_1$

\begin{center}
$\{c,e,r,f,d,h,w\}, \{m,p,t\}$
\end{center}

\noindent - $P_{sent}$ describes the probability for a word to belong to a given sentence (5 categories)

\noindent - $P_C$ is adapted from $P_2$ so that \textit{'fear'} belongs to \textit{Verb} and \textit{Noun}

\begin{center}
$\{c,r,m,t,d,w,f(\frac{1}{2})\}, \{e,p,h,f(\frac{1}{2})\}$
\end{center}


\noindent - $P_D$ is adapted from $P_2$ and creates a special category for \textit{'fear'}

\begin{center}
$\{c,r,m,t,d,w\}, \{e,p,h\}, \{f\}$
\end{center}

\noindent - $P_{posi}$ describes the probability for a word to appear in a given position (3 categories)

\medskip

\begin{figure}[!t]
\centering
\includegraphics[scale=0.5]{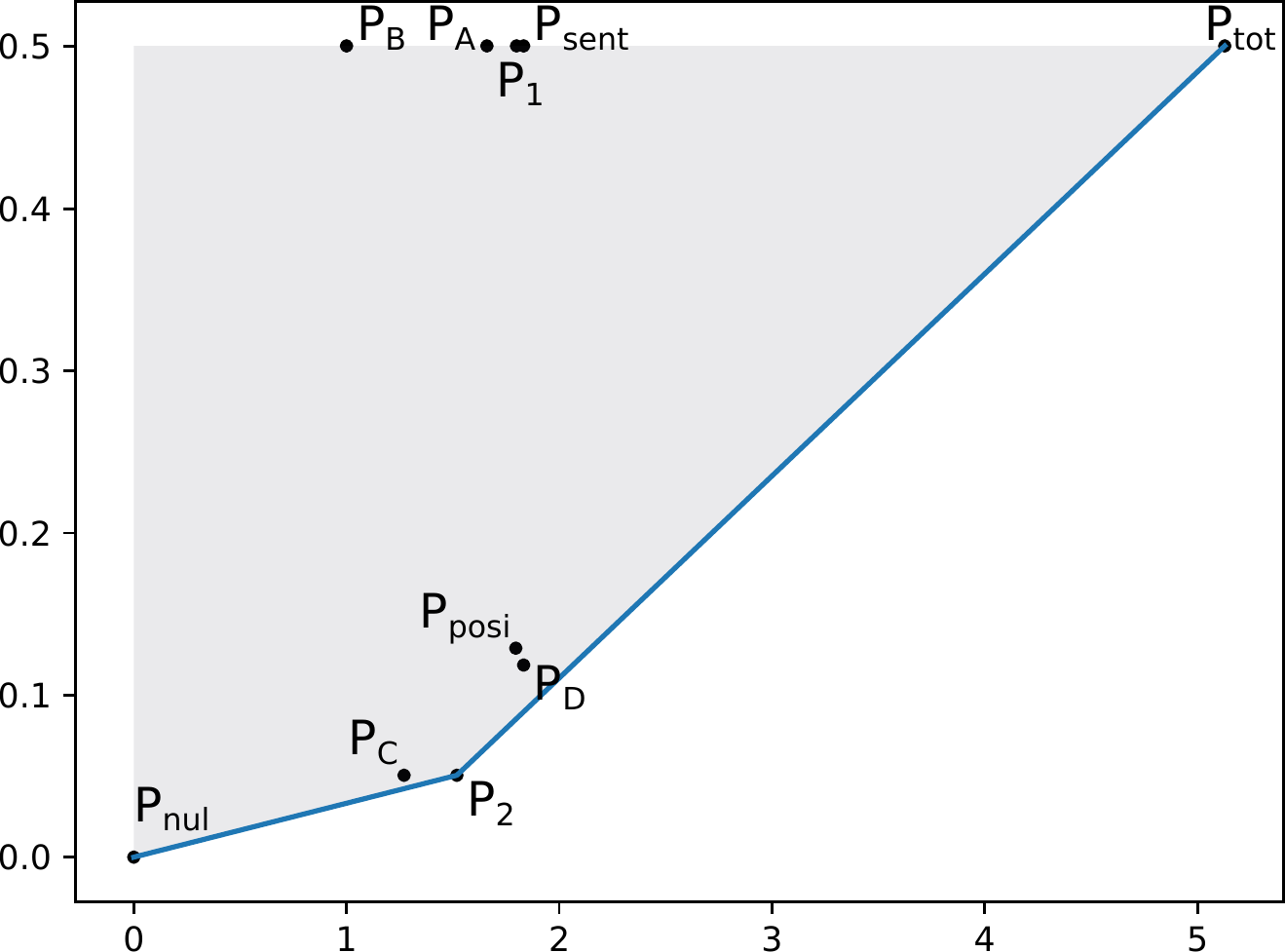}
\captionof{figure}{$I_{L}(P;P_B)$ represented w.r.t. $I_{L}(P)$ for different partitions: acceptable solutions of program~(\ref{eq:2}) lies on the convex hull boundary of the set of all partitions. Solution for $\mu$ is given by the tangent of slope $\mu$. Non-trivial solution is $P_2$.}
\label{fig_MI}
\end{figure}
Acceptable solutions of (\ref{eq:1}) and (\ref{eq:2}) are, respectively, on the convex hull boundary in Fig.\ref{fig_deltaI} and Fig.\ref{fig_MI}. While the lowest parameter (non trivial) solutions are $P_B$ for context and $P_2$ for syntax, one can check that partitions $P_1$, $P_A$ and $P_{sent}$ are all close to the boundary in Fig.\ref{fig_deltaI}, and that partitions $P_C$, $P_D$ and $P_{posi}$ are all close to the boundary in Fig.\ref{fig_MI}, as expected considering their information content. 
\end{example}

\subsection{Experiments}
In this section we illustrate the emergence of syntactic information via the application of objectives (\ref{eq:1}) and (\ref{eq:2}) to a natural language corpus. We show that the information we acquire indeed translates into known syntactic and contextual categories.


For this experiment we created a corpus from the Simple English Wikipedia dataset \cite{kauchak13}, selected along three main topics: \emph{Numbers}, \emph{Democracy}, and \emph{Hurricane}, with about $430$ sentences for each topic and a vocabulary of $2963$ unique words. The stochastic language is the set $L^3$ of $3$-gram frequencies from the dataset. In order to avoid biases with respect to the final punctuation, we considered overlapping $3$-grams over sentences.
For the sake of evaluation, we construct one contextual and one syntactic embedding for each word. These are the probabilistic partitions over gold standard contextual and syntactic categories. The contextual embedding $P_{con}$ is defined by relative frequency in the three topics. The results for this partition are $I_{L^3}(P_{con})=0.06111$ and $I_{\overline{L^3}}(P_{con})=0.06108$, corresponding to a $\gamma$ threshold of $6.22.10^{-4}$ in (\ref{eq:1}), and thus distribution over topics can be considered as an almost purely contextual partition. The syntactic partition $P_{syn}$ is the distribution over POS categories (tagged with the Stanford tagger, \citet{stanfordtagger}).

Using the gold categories, we can manipulate the information in the partitions by merging and splitting across contextual or syntactic categories. We study how the information calculated by (\ref{eq:1}) and (\ref{eq:2}) evolve; we  validate our claims if we can deduce the nature of information from these statistics.

\begin{figure}[!ht]
\centering
\includegraphics[scale=0.5]{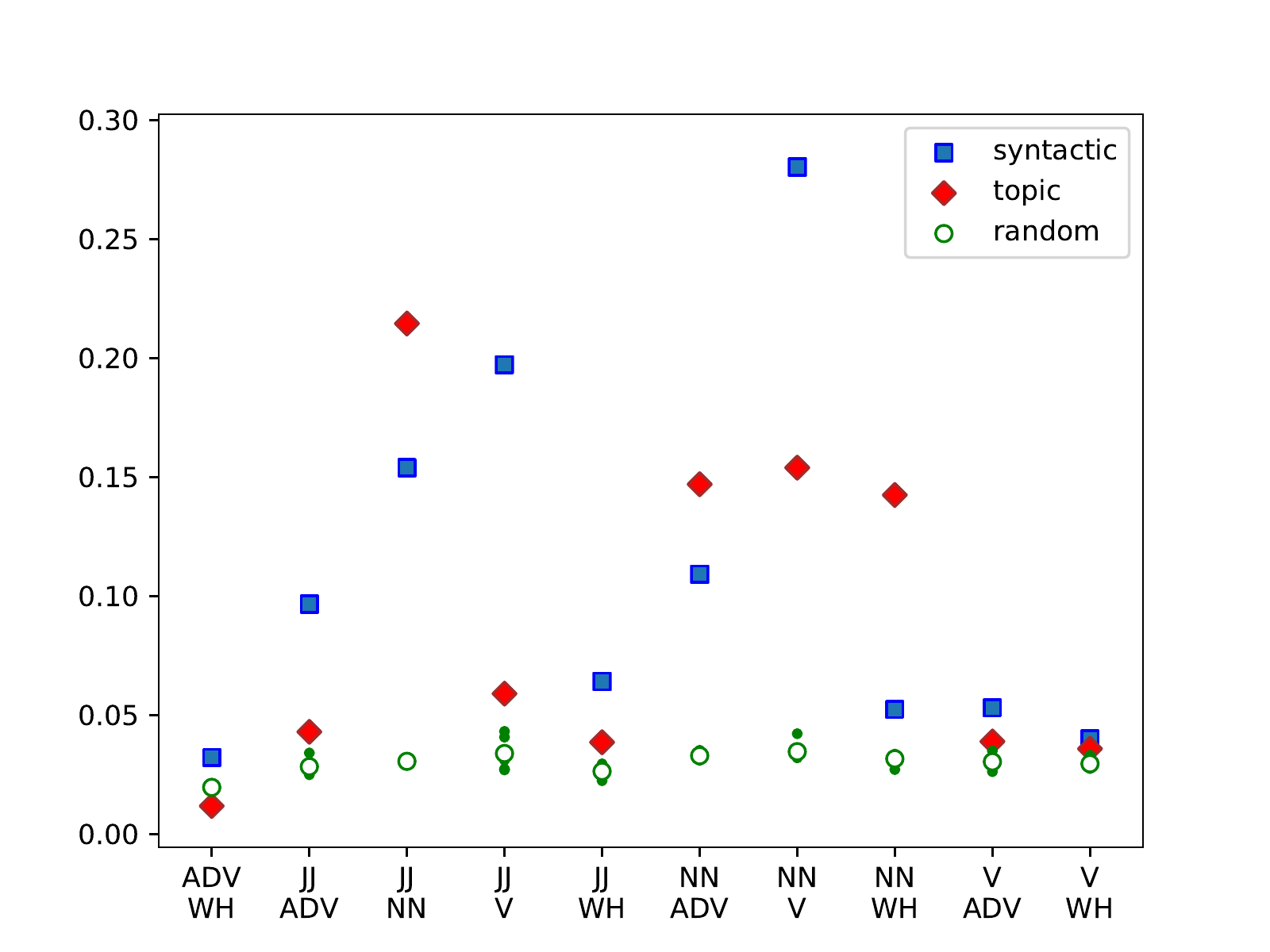}
\captionof{figure}{Increase of information $\Delta_I$ in three scenarios: syntactic split, topic split and random split.}
\label{fig_info}
\end{figure}

We start from the syntactic embeddings and we split and merge over the following POS categories: Nouns (\emph{NN}), Adjectives (\emph{JJ}), Verbs (\emph{V}), Adverbs(\emph{ADV}) and Wh-words (\emph{WH}). For a pair of categories (say \emph{NN+V}), we create:
\begin{itemize}
\item $P_{merge}$ merges the two categories ($NN+ V$)
\item $P_{syntax}$ splits the merged category into $NN$ and $V$ (syntactic split)
\item $P_{topic}$ splits the merged category into $(NN+ V)_{t_1}$, $(NN+ V)_{t_2}$ and $(NN+ V)_{t_3}$ along the three topics (topic split)
\item $P_{random}$ which splits the merged category into $(NN+ V)_1$ and $(NN+ V)_2$ randomly (random split)
\end{itemize}
It is clear that each split will increase the information compared to $P_{merge}$. We display the simple information gains $\Delta_I$ in Fig.\ref{fig_info}. The question is whether we can identify if the added information is syntactic or contextual in nature, i.e. if we can find a $\mu$ for which the $\mu, \gamma$-syntactic program (\ref{eq:2}) selects every syntactic splitting and rejects every contextual or random one.

\begin{figure}[!ht]
\centering
\includegraphics[scale=0.5]{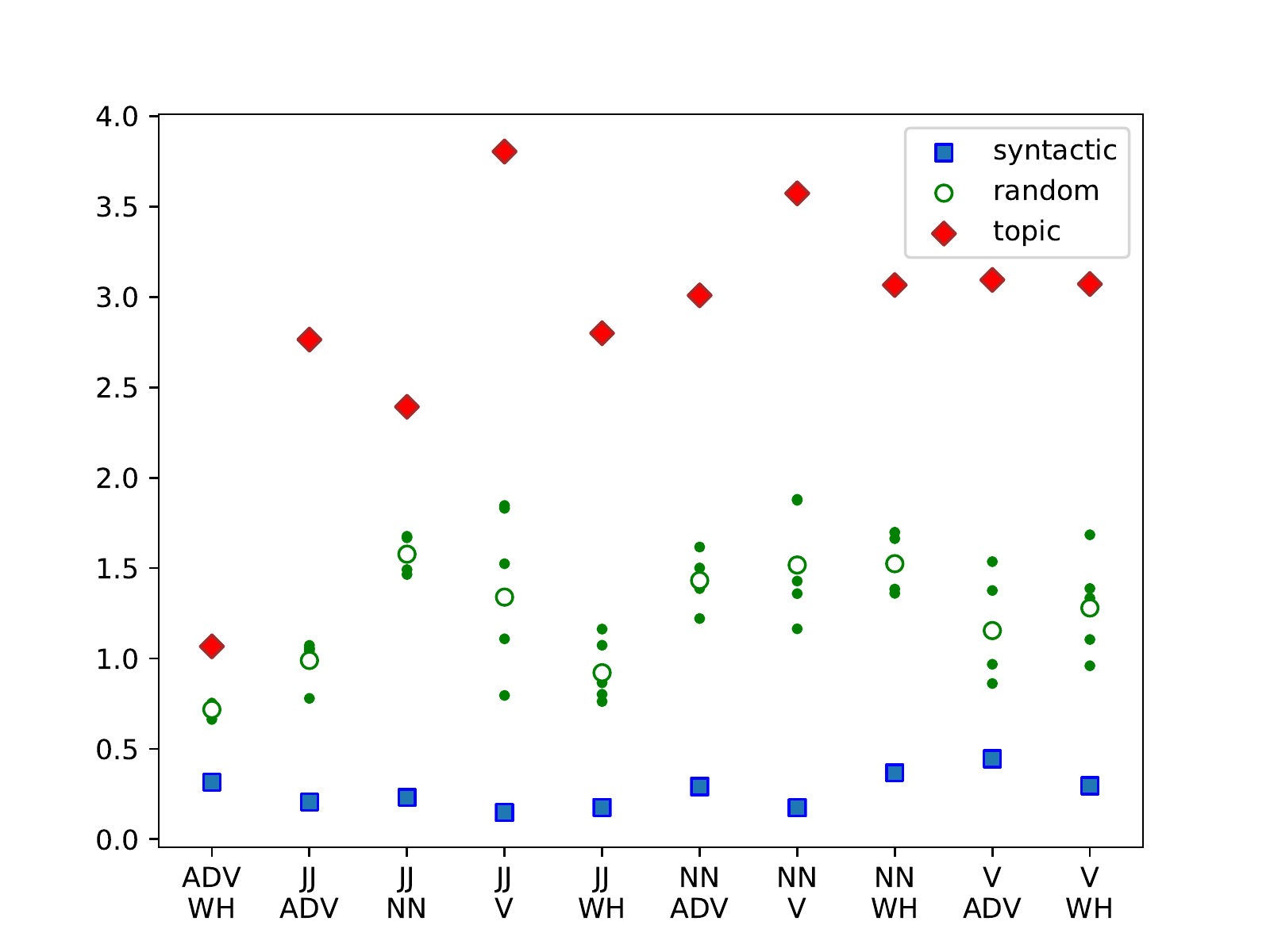}
\captionof{figure}{Ratio $\Delta_{MI}/ \Delta_I$ in three scenarios: syntactic split, topic split and random split. Considering objective (\ref{eq:2}) with parameter $\mu = 0.5$ leads to discrimination between contextual and syntactic information.}
\label{fig_mu}
\end{figure}

Fig.\ref{fig_mu} represents  the ratio between the increase of mutual information (relatively to  $P_{con}$) $\Delta_{MI}$ and the increase of information $ \Delta_I$, corresponding to the the threshold $\mu$ in (\ref{eq:2}). It shows that indeed for a $\mu = 0.5$ syntactic information (meaningful refinement according to POS) will be systematically selected, while random or topic splittings will not. We conclude that even for a small natural language sample, syntactic categories can be identified based on statistical considerations, where a language model learning algorithm would need further information or hypotheses.


\subsection{Integration with Models}

We have shown that our framework allows to search for syntactic categories without prior hypothesis of a particular model. Yet if we do have a hypothesis, we can indeed search for the syntactic categories that fit the particular class of models $\mathcal{M}$. In order to find the categories which correspond to the syntax rules that can be formulated in a given class of models, we can integrate the model class in the training objective by replacing entropy by the negative log-likelihood of the training sample.

Let $M \in \mathcal{M}$ be a model, which takes a probabilistic partition $P = (V,C, \pi)$ as input, and let $LL(M, P, L_S)$ be the log-likelihood obtained for sample $S$. We will denote 

$$\Tilde{H}(L_S, P) = - \sup_{M \in \mathcal{M}} LL(M, P, L_S)$$
$$\Tilde{I}_{L_S}(P) = \Tilde{H}(\overline{\overline{L_S}}, P) - \Tilde{H}(L_S, P)$$
Following Definition~\ref{def:mutual-info}, we define
$$\Tilde{I}_{L_S}(P;P') = \Tilde{H}(L_S, P)+ \Tilde{H}(L_S, P') - \Tilde{H}(L_S, P\cdot P')$$

We may consider the following program:

\begin{itemize}
\item A partition $P$ is said to be \emph{$\gamma$-contextual} if it minimizes $$\Tilde{I}_{L_S}(P) (1- \gamma) - \Tilde{I}_{\overline{L_S}}(P)$$
\item Let $P_{*}$ be a $\gamma$-contextual partition for $L$, $\mu \in \mathbb{R^+}$, $k\in \mathbb{N}$. The partition $P$ is considered $\mu, \gamma$-syntactic if it minimizes $$\max_{P^*}\Tilde{I}_{L_S}(P;P^*) - \mu~\Tilde{I}_{L_S}(P)$$
\end{itemize}


\section{Conclusion and Future Work}\label{concl}

In this paper, we proposed a theoretical reformulation for the problem of learning syntactic information from a corpus. Current language models have difficulty acquiring syntactically relevant generalizations for diverse reasons. On the one hand, we observe a natural tendency to lean towards shallow contextual generalizations, likely due to the maximum likelihood training objective. On the other hand, a corpus is not representative of human linguistic competence but of performance. It is however possible for linguistic competence - syntax - to emerge from data if we prompt models to establish a distinction between syntactic and contextual (semantic/pragmatic) information. 



Two orientations can be identified for future work. The immediate one is experimentation. The current formulation of our syntax learning scheme needs adjustments in order to be applicable to real natural language corpora. 
At present, we are working on an incremental construction of the space of categories.

The second direction is towards extending the approach to morphologically rich languages. In that case, two types of surface realization need to be considered: word order and morphological markers. An agglutinating morphology probably allows a more straightforward application of the method, by treating affixes as individual elements of the vocabulary. The adaptation to other types of morphological markers will necessitate more elaborate linguistic reflection.
 




\bibliography{acl2020}
\bibliographystyle{acl_natbib}

\end{document}